\documentclass[a4paper,english,final]{article}

\usepackage[utf8]{inputenc}
\usepackage[T1]{fontenc}
\usepackage[english]{babel}
\usepackage{graphicx}

\usepackage{textcomp}

\usepackage{cmap}

\usepackage[usenames,svgnames,hyperref]{xcolor}

\usepackage{setspace}
\usepackage{hyperref}
\hypersetup{
  a4paper,
  bookmarksnumbered,
  bookmarksopen,
  bookmarksopenlevel=0,
  colorlinks,
    citecolor=teal,
    linkcolor=Crimson,
	pdfauthor={},
	pdftitle={},
	pdfsubject={},
	pdfkeywords={},
}

\usepackage{verbatim}

\usepackage{mathtools}
\usepackage[thmmarks,amsmath,hyperref]{ntheorem}

\newtheorem{sectioncounter}{}[section]

\theoremsymbol{\ensuremath{_\Box}}

\theorembodyfont{\upshape}
\newtheorem{definition}[sectioncounter]{Definition}
\newtheorem{example}[sectioncounter]{Example}
\newtheorem{remark}[sectioncounter]{Remark}

\theorembodyfont{\itshape}
\newtheorem{theorem}[sectioncounter]{Theorem}
\newtheorem{lemma}[sectioncounter]{Lemma}
\newtheorem{proposition}[sectioncounter]{Proposition}
\newtheorem{corollary}[sectioncounter]{Corollary}

\newcommand{\conc}[1]{\ensuremath{\mathsf{#1}}}

\theorembodyfont{\upshape}
\theoremsymbol{\ensuremath{_\blacksquare}}
\newtheorem*{proof}{Proof}
\DeclareMathOperator{\naf}{\textbf{not}}

\usepackage[charter,cal=cmcal]{mathdesign}

\usepackage[sort]{cleveref}

\crefname{section}{Section}{sections}
\crefname{subsection}{Section}{sections}
\crefname{subsubsection}{Section}{sections}

\crefname{definition}{Definition}{definitions}
\crefname{theorem}{Theorem}{theorems}
\crefname{lemma}{Lemma}{lemmata}
\crefname{propositions}{Proposition}{propositions}
\crefname{corollary}{Corollary}{corollaries}
\crefname{property}{Property}{properties}

\crefname{example}{Example}{examples}
\crefname{problem}{Problem}{problems}
\crefname{notation}{Notation}{Notation}
\crefname{algo}{Algorithm}{algorithms}

\title{Generating Ontologies from Templates: A Rule-Based Approach for Capturing Regularity}
\author{Henrik Forssell \and Christian Kindermann \and Daniel P. Lupp
  \and Uli Sattler \and 
Evgenij Thorstensen}
\date{Technical Report}

\newcommand{\eval}{\mathsf{eval}}
\newcommand{\expand}{\mathsf{Exp}}
\newcommand{\oexpand}{\mathsf{1Exp}}
\newcommand{\lang}{\mathcal{L}}
\newcommand{\ont}{\mathcal{O}}
\newcommand{\isa}{\sqsubseteq}

\newcommand{\E}{\exists}
\newcommand{\A}{\forall}

\begin{document}

\maketitle
\tableofcontents
\newpage
\begin{abstract}
  We present a second-order language that can be used to succinctly
  specify ontologies in a consistent and transparent manner. This
  language is based on ontology templates (OTTR), a framework for
  capturing recurring patterns of axioms in ontological modelling. The
  language, and our results are independent of any specific DL.

  We define the language and its semantics, including the case of
  negation-as-failure, investigate reasoning over ontologies specified
  using our language, and show results about the decidability of
  useful reasoning tasks about the language itself. We also state and
  discuss some open problems that we believe to be of interest.
\end{abstract}

\section{Introduction}
\label{introduction}
The phenomenon of frequently occurring structures in ontologies engineering (OE) has received attention from a variety of angles. One of the first accounts is given in \cite{DBLP:conf/ekaw/Clark08}, where repeated versions of general conceptual models are identified. Similar observations gave rise to the notion of \textit{Ontology Design Patterns} (ODP) as abstract descriptions of best practices in OE \cite{DBLP:conf/semweb/Gangemi05,DBLP:conf/iceis/BlomqvistS05,DBLP:books/ios/HGJKP2016}. Another view, emphasizing common ontological distinctions, led to the emergence of \textit{Upper Ontologies} which aim to categorize general ideas shareable across different domains \cite{DBLP:conf/ijcai/GangemiGMO01}. Orthogonal to such conceptual patterns, the existence of syntactic regularities in ontologies has been noted and some aspects of their nature have been analyzed \cite{DBLP:conf/owled/MikroyannidiMIS12,DBLP:conf/semweb/MikroyannidiISR11,DBLP:conf/ekaw/MikroyannidiQTFSP14}.

In this paper, we propose a new language that allows expressing
patterns of repeated structures in ontologies. This language is
rule-based and has both a model-theoretic and a fixpoint semantics, for which we show that they coincide. In contrast to other rule languages ``on top of'' DLs, in this language,  firing a rule results in the addition of TBox and/or ABox axioms, with the goal to 
succinctly describe ontologies, thereby making them more readable and maintainable. 

%Motivation: syntactic regularity and conceptually related axioms.
%The phenomenon of frequently occurring patterns of axioms in ontologies has prompted the development of mechanisms for capturing and expressing such patterns explicitly. One of the first proposals introduces the idea of macros to formally specify reusable patterns \cite{vrandevcic2005explicit}. Recently, this conceptual idea has been made more concrete under the name of ontology templates \cite{dl-templates}, OTTR for short.

%We make use of this duality to create a template language for specifying and extending ontologies. In contrast to typical mapping approaches used in e.g. OBDA and data integration \cite{}, the language has a recursive semantics similar to datalog, where rules are applied exhaustively until a fixpoint is reached. 

%\medskip % subsection{Motivation}

Given that DL ontologies are sets of axioms, an ontology provides no means to arrange its axioms in a convenient manner for ontology engineers. In particular, it is not possible to group conceptually related axioms or indicate interdependencies between axioms. While ontology editors such as Prot\'{e}g\'{e}\footnote{https://protege.stanford.edu/} display an ontology through a hierarchy of its entities, conceptual interdependencies between axioms are hidden and the underlying structural design of an ontology remains obfuscated.

\begin{example}
    \label{introduction:example:simplePattern}
    Consider the ontology
\begin{align}
            \ont_1 = \{\conc{Jaguar} &\isa \conc{Animal},&\conc{Jaguar} &\isa \A \conc{hasChild}.\conc{Jaguar},\\
            \conc{Tiger} &\isa \conc{Animal},&\conc{Tiger} &\isa \A \conc{hasChild}.\conc{Tiger},\\
            \conc{Lion} &\isa \conc{Animal},&\conc{Lion} &\isa \A \conc{hasChild}.\conc{Lion}\}
\end{align}
    Then, an ontology editor will group the entities $\conc{Jaguar}, \conc{Tiger}$ and $\conc{Lion}$ under $\conc{Animal}$ according to their class hierarchy.

% I guess people at DL know how Protege looks liks    
% \begin{figure}[h]
% \begin{center}
% \includegraphics[width=3cm]{../pic/hierarchy.png}
% \caption{Ontology view in Prot\'{e}g\'{e}}
% \end{center}
% \end{figure}

However, $\ont_1$ contains no indication that every subclass $X$ of $\conc{Animal}$ can have only children of the same class $X$. Assume this regularity is no coincidence but a desired pattern that should hold for any subclass of $\conc{Animal}$. Currently, ontology engineers have no means of expressing or enforcing such a pattern other than dealing with the ontology as a whole, inspecting all axioms separately, and making necessary changes manually.
\end{example}

Expressing patterns such as in Example
\ref{introduction:example:simplePattern} explicitly has a potential to
reveal some aspects of the intentions for the design of an ontology.
%As a consequence, ontology comprehension and readability may be improved.
%Furthermore, a mechanism enforcing such patterns automatically may improve ontology maintenance.

\begin{example}
    \label{introduction:example:generator}
    Consider the ontology
    \begin{align*}
        & \ont_2 = \{\conc{Jaguar} \isa \conc{Animal},\ 
            \conc{Tiger} \isa \conc{Animal},\ 
        \conc{Lion} \isa \conc{Animal}\}
    \end{align*}
    In addition, consider the rule $$g \colon \underbrace{\{?X \isa \conc{Animal}\}}_{\text{Body}} \rightarrow \underbrace{\{?X \isa \A \conc{hasChild}.?X \}}_{\text{Head}},$$ where $?X$ is a variable. We can interpret the body of this rule as a query which, when evaluated over the ontology $\ont_2$, returns substitutions for  $?X$. These substitutions  can then be used to instantiate the  axioms in the head of the rule. Firing the above rule over $\ont_2$ would  add all those resulting axioms to $\ont_2$, thereby reconstructing $\ont_1$ from Example \ref{introduction:example:simplePattern}.

    %(Should I? What about dropping ``body/head'' speak for something more descriptive, e.g. ``query/production''?).
\end{example}

In the following, we will call such rules \textit{generators}. The possible benefits of  generators are threefold.
Firstly, $\ont_2$ in combination with $g$ is easier to understand
because $g$ makes a statement about all subconcepts of $\conc{Animal}$
that \textit{the type of an animal determines the type of its
  children}. This is a kind of meta-statement about concepts which a
user of an ontology can usually only learn by inspecting (many) axioms in an ontology.
Secondly, $\ont_2$ in combination with $g$ is easier to maintain and extend compared to $\ont_1$, where a user would have to manually ensure that the meta-statement continues to be satisfied after new concepts have been added.  
Thirdly, conceptual relationships captured in a generator such as $g$ are easy to reuse and can foster interoperability between ontologies in the spirit of ontology design patterns.

We close this section with more elaborate examples to demonstrate the benefits generators such as $g$ can provide.

\subsection{Examples}

\begin{example}[Composition]
    \label{introduction:example:composition}
    Assume we want to model typical roles in groups of social predatory animals. One such a role would be that of a hunter. A challenge for representing such knowledge is that different collective nouns are used for different animals, e.g. a group of lions is called a ``pride'', a group of wild dogs is called a ``pack'', a group of killer whales is called a ``pod'', etc. Therefore, a mechanism that can conveniently iterate over all these group formations would be beneficial.

    Consider the following query $Q_1$:
    
        \begin{align}
            Q_1= \{?X & \isa \conc{Animal}, \label{Q1:animal}\\
                ?X &\isa \E \conc{eats}.\conc{Animal}, \label{Q1:carnivore}\\
            ?X &\isa \E \conc{hunts}.\conc{Animal}, \label{Q1:predator}\\
            ?Y &\isa \conc{SocialGroup}, \label{Q1:group}\\
            ?X &\isa \E \conc{socialisesIn}.?Y, \label{Q1:association}\\
            ?Y &\isa \E \conc{hasMember}.?X,\\
            &\conc{socialisesIn} \equiv \conc{hasMember}^{-} \label{Q1:inverse}\}
        \end{align}
        Lines \ref{Q1:animal}--\ref{Q1:predator} bind the variable $?X$ to a predatory animal. Line \ref{Q1:group} binds the variable $?Y$ to a type of social group and lines \ref{Q1:association}--\ref{Q1:inverse} associate a particular type of animal with its respective social group. Given the bindings for $?X$ and $?Y$ it is straightforward to express that a particular type of predator $?X$ is a hunter in its respective social group, namely: $?Y \isa \E \conc{hasHunter}.?X$. A generator such as in Example~\ref{introduction:example:generator} could capture this relationship:
        $$g_1 \colon Q_1 \rightarrow \{?Y \isa \E \conc{hasHunter}.?X \}$$
\end{example}

\begin{example}[Extension]
    \label{introduction:example:extension}

Extending generator $g_1$ from Example~\ref{introduction:example:composition} to capture more specialised knowledge is straightforward. Consider predatory ants of the family \conc{Formicidae}. These ants generally live in colonies with an elaborate social organisation consisting of workers, drones, queens, etc.

First, we extend query $Q_1$ with the following axioms:
        \begin{align}
            Q_2 = Q_1 \cup \{?X &\isa \conc{Formicidae}, \label{Q2:formicidae}\\
            ?Z &\isa ?Y \label{Q2:subgroupOfGroup}\\
            ?X &\isa \E \conc{socialisesIn}.?Z\}\label{Q2:scope}
        \end{align}
        Axiom \ref{Q2:formicidae} requires $?X$ to bind to a type of $\conc{Formicidae}$, e.g. $?X = \conc{SafariAnt}$. According to query $Q_1$, the variable $?Y$ binds to a general $\conc{SocialGroup}$, e.g. $?Y = \conc{AntColony}$. Then, axiom \ref{Q2:subgroupOfGroup} binds $?Z$ to a more specialised subgroup of a $?Y$. Finally, axiom \ref{Q2:scope} ensures that this subgroup $?Z$ is associated with $?X$. So for $?X = \conc{SafariAnt}$ we get $?Z = \conc{SafariAntColony}$.

        Next, we can specify the generator to add all desired axioms based on matches of query $Q_2$ specialised for ants:
        \begin{align*}
            g_2 \colon Q_2 &\rightarrow\\
            \{?Z &\isa \E \conc{hasHunter}.?X,\\
            ?Z &\isa \E \conc{hasWorker}.?X,\\
        ?Z &\isa  \E \conc{hasDrone}.?X,\\
    ?Z &\isa \E \conc{hasQueen}.?X \}
        \end{align*}
        Note how the body and head of generator $g_1$ from Example~\ref{introduction:example:composition} have been reused and extended only by set unions.
\end{example}

\begin{example}[Negative Guards]
    \label{introduction:example:negativeGuards}
Often, general relationships are subject to exceptions. While most ants hunt and feed cooperatively, there are some genera of ants, e.g. Myrmecia, that do not. Therefore, $g_2$ in Example~\ref{introduction:example:extension} would generate an undesired axiom, namely $\conc{MyrmeciaAntColony} \isa \conc{hasHunter}.\conc{MyrmeciaAnt}$. This motivates guards in the body of generators that may not only specify positive constraints but also negative ones:

\begin{align*}
Q_3 = Q_2 \cup \{\naf~~?X &\isa \conc{MyrmeciaAnt},\\
\naf~~?Z &\isa \conc{MyrmeciaAntColony}\}
\end{align*}
        \begin{align*}
            g_3 \colon Q_3 &\rightarrow\\
            \{?Z &\isa \E \conc{hasHunter}.?X,\\
            ?Z &\isa \E \conc{hasWorker}.?X,\\
        ?Z &\isa  \E \conc{hasDrone}.?X,\\
    ?Z &\isa \E \conc{hasQueen}.?X \}
        \end{align*}

        One might argue that the effect of negative guards could also be achieved by positive guards using negated concepts in DL, i.e. $?X \isa \lnot \conc{MyrmeciaAnt}$ instead of $\naf ?X \isa \conc{MyrmeciaAnt}$. However, this approach would necessitate the introduction of a potentially large number of axioms of type $?X \isa \lnot \conc{MyrmeciaAnt}$ in the given ontology. This can be avoided by using $g_3$.

        Another advantage of negative guards is the possibility to explicitly express default assumptions for lack of better knowledge. An ant colony of a certain genus usually consists of only ants of this genus, e.g.
        
        \begin{equation}
            \conc{SafariAntColony} \isa \A \conc{hasMember}.\conc{SafariAnt}. \label{P5:universal}
    \end{equation}
        
    However, some genera of ants are social parasites that enslave other ant species. In such a case, the default assumption about the homogeneity of an ant colony is wrong and the axiom \ref{P5:universal} should not be added. 
        \begin{align*}
            Q_4 = \{?X &\isa \conc{Ant},\\
            ?Y &\isa \conc{AntColony},\\
            ?Y &\isa \E \conc{hasMember}.?X,\\
            ?Z &\isa \conc{Ant},\\
                ?X &\isa \lnot ?Z,\\
                \naf~~?X &\isa \E \conc{enslaves}.?Z,\\
                \naf~~?Y &\isa \E \conc{hasMember}.?Z,\}
        \end{align*}
        \begin{align*}
            g_4 \colon Q_4 \rightarrow \{?Y \isa \A \conc{hasMember}.?X\}
        \end{align*}

\end{example}

\begin{example}[Recursion]
    Contagious diseases may be transmitted between animals sharing a habitat.
    Overlapping habitats of infected animals may result in a propagation of diseases across habitats.
\begin{figure}[h]
\begin{center}
\includegraphics[width=8cm]{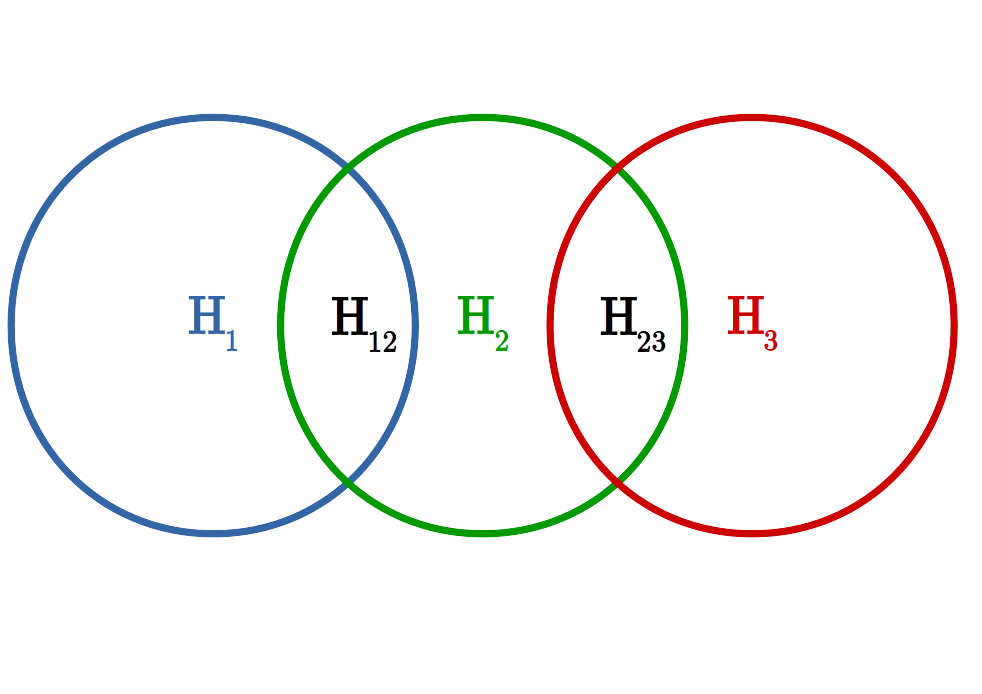}
\caption{Overlapping Habitats}
\label{overlap}
\end{center}
\end{figure}

Assume there is an overlap between habitats $H_1, H_2, H_3$ such that there is no overlap between $H_1$ and $H_3$, $H_{12}$ describes the overlap between $H_1$ and $H_2$, and $H_{23}$ describes the overlap between $H_2$ and $H_3$ (see Figure~\ref{overlap}). Then, a disease infected animal living in $H_1$ may affect an animal in $H_2$ which in turn may affect an animal in $H_3$. Such an iterative process may be captured by repeatedly applying a single generator. 

Consider the following query:

\begin{align}
    Q_5 = \{?X &\isa \conc{Animal},\\
    ?Y &\isa \conc{Animal},\\
    ?D &\isa \conc{ContagiousDisease},\\
    ?H &\isa \conc{Habitat},\\
    ?X &\isa \E \conc{suffersFrom}.?D, \label{Q:recursive:suffering}\\
    ?Y &\isa \A \conc{isSusceptibleTo}.?D, \label{Q:recursive:susceptible}\\
    ?X &\isa \E \conc{livesIn}.?H,\label{Q:recursive:sharedHabitat1}\\
        ?Y &\isa \E \conc{livesIn}.?H\label{Q:recursive:sharedHabitat2}\}
\end{align}

Axioms \ref{Q:recursive:suffering} and \ref{Q:recursive:susceptible} express the requirements for a disease to be transmitted between animals while axioms \ref{Q:recursive:sharedHabitat1} and \ref{Q:recursive:sharedHabitat2} capture the requirement of a shared environment. Using query $Q_5$, we can represent the propagation of a disease between animals across habitats:
\begin{align*}
    g_5 & : Q_5 \rightarrow \{?Y \isa \E \conc{suffersFrom}.?D \}
\end{align*}
Clearly, the generation of an instance of $?Y \isa \E \conc{suffersFrom}.?D$ could yield a new match for $Q_5$ in the body of $g_5$. Therefore, generator $g_5$ has to be applied repeatedly until a fixpoint is reached.

\end{example}
        
\begin{example}[Encapsulation]
    Inspecting the queries $Q_1, Q_2$, and $Q_3$ in Examples~\ref{introduction:example:composition}--\ref{introduction:example:negativeGuards}, it is apparent that different parts in the queries correspond to different conceptual ideas. For example, in query $Q_1$ the axioms can be grouped into ones about predators and others about social groups. Such a grouping would provide valuable information for an ontology engineer to indicate conceptual relationships between certain sets of axioms:
    \begin{equation*}
            \left.
        \begin{aligned}
            \{?X & \isa \conc{Animal},\\
                ?X &\isa \E \conc{eats}.\conc{Animal},\\
            ?X &\isa \E \conc{hunts}.\conc{Animal}\}
        \end{aligned}
    \right\} \text{Predator}
    \end{equation*}
    \begin{equation*}
            \left.
        \begin{aligned}
            \{?Y &\isa \conc{SocialGroup}, \\
            ?X &\isa \E \conc{socialisesIn}.?Y,\\
            ?Y &\isa \E \conc{hasMember}.?X,\\
            &\conc{socialisesIn} \equiv \conc{hasMember}^{-}\}
        \end{aligned}
    \right\} \text{Social Group}
    \end{equation*}
Reasonable ontology templates \cite{OTTR-ISWC18,dl-templates}, OTTR for short,
introduced a framework for indicating such conceptual relationships. A
template is defined as a named ontology with a set of variables. The variables can be instantiated with concept and role expressions to yield a set of valid axioms. Moreover, templates may be composed to give rise to more complex templates. Choosing intention-revealing names for templates and composing appropriately named templates may improve ontology comprehension by making the structural design of an ontology visible.

A template, i.e. a set of axioms with variables, can also be interpreted as a query, asking for concept and role expressions in an existing ontology that match the pattern represented by the template. These expressions can then, in principle, be fed into a different template to produce new axioms. This idea captures conceptual interdependencies between templates or, more generally, axiomatic patterns.

Clearly, it is straightforward to integrate OTTR as part of a preprocessing step into our rule language. This has  not only the potential to foster the reuse of conceptually related set of axioms in an intention-revealing manner, but can also to further improve the maintainability of generators by the principle of information hiding. A change in a template will be propagated automatically to all instances of the use of the template.
\end{example}

%%% Local Variables:
%%% mode: latex
%%% TeX-master: "main-techrep"
%%% End:

\section{Preliminaries}
Let  $N_I$, $N_C$, and $N_R$ be sets of \emph{individual},
\emph{concept}, and \emph{role names}, each containing a distinguished subset of \emph{individual}, \emph{concept},
and \emph{role variables} $V_I$, $V_C$, and $V_R$. A \emph{concept}
(resp. \emph{role}) is either a concept name (resp. role name) or a
concept expression (resp. role expression) built using the usual DL
constructors \cite{DBLP:conf/dlog/2003handbook}. Since we do not distinguish between TBoxes
and ABoxes, an \emph{axiom} is either an assertion of the form
$\conc{C}(\conc{a})$ or $\conc{R}(\conc{a},\conc{b})$ for a
concept $\conc{C}$, role $\conc{R}$, and individual names $\conc{a},\conc{b}$ or an inclusion statement
$\conc{C}\isa \conc{D}$ for concepts or roles $\conc{C}$ and $\conc D$. 
A \emph{theory} is a (possibly infinite) set of axioms, whereas an
\emph{ontology} is a finite set of axioms. % As concepts, roles, axioms, and ontologies may in general contain variables, we
% refer to those without variables as \emph{ground
%   concepts}, \emph{ground roles}, \emph{ground axioms}, and \emph{ground ontologies}. 
A set $\lang$ of
individuals, concepts, and roles is called a \emph{language}.

A \emph{template} $T$ is an ontology, and we write $T(V)$ for $V\subseteq
V_I\cup V_C\cup V_R$ the set of variables  occurring in $T$. For the sake of brevity, we occasionally omit the variable set $V$ when it is either clear from context or nonvital to the discussion. Templates can
be \emph{instantiated} by applying a substitution to them. A
\emph{substitution} $\sigma$ is a function that maps individual,
concept, and role variables to individuals, concepts, and roles
respectively. We require that substitutions respect the type of a variable, so that the result of instantiating a template is a well-formed ontology.
%
% Templates can be \emph{instantiated} by applying a substitution
% $\sigma$ to them, which we denote by $T\sigma$. Substitutions can be
% restricted in various ways, e.g., only instantiating variables with
% names. 
For $\lang$ a language, an \emph{$\lang$-substitution} is one whose
range is a subset of $\lang$. The \emph{$\lang$-evaluation of $T$
  over $\ont$}, written $\eval(T, \ont, \lang) $, is the set of substitutions
defined as follows:
$$\eval(T, \ont, \lang) = \{ \sigma \text{ an $\lang$-substitution}
\mid \ont \models T\sigma \},$$ 
where $T\sigma$ is the instantiation of $T$ with $\sigma$. 
Furthermore, we define
$\eval(\emptyset,\ont,\lang)$ to be the set of all $\lang$-substitutions.

Finally, we say that an ontology $\ont$ \emph{is weaker than} $\ont'$ if $\ont' \models \ont$, and \emph{strictly weaker} if the reverse does not hold.

%%% Local Variables:
%%% mode: latex
%%% TeX-master: "main-techrep"
%%% End:

\section{Generators and GBoxes}

%\subsection{Syntax and semantics}
\label{sec:syntax-semantics}

In this section we define the syntax and semantics of generators and GBoxes and discuss
some examples. 

\begin{definition}
    \label{generator}
  A \emph{generator} $g$ is an expression of the form
  $T_B(V_B)\rightarrow T_H(V_H)$,  for $T_B(V_B), T_H(V_H)$  templates
  with $V_H \subseteq V_B$. $T_B$ and $T_H$ are respectively called
  the \emph{body} and \emph{head} of $g$, and we write $B(g)$ and $H(g)$ to denote them.
\end{definition}

\begin{example}
  \label{ex:animal-generator}
  $g \colon \{?X \isa \conc{Animal}\} \rightarrow \{?X \isa \forall \conc{hasChild}.?X\}$ is a generator, with a single variable $?X$. 
\end{example}

Next, we define the semantics for generators and sets of generators
based on entailment to ensure that generators behave independent of
the syntactic form of an ontology. In this choice we diverge from the work done on OTTR \cite{OTTR-ISWC18}, as OTTR template semantics is defined syntactically.

\begin{definition}
\label{sec:generator:satisfaction}
Let $g \colon T_B(V_B)\rightarrow T_H(V_H)$
 be a generator. 
 A theory $\ont$ \emph{satisfies $g$ wrt.} $\lang$ if, for every $\lang$-substitution $\sigma$ such that $\ont \models T_B\sigma$, we have $\ont \models T_H\sigma$.
\end{definition}

\begin{example}
  Consider the generator $g$ from \Cref{ex:animal-generator}. The theory $\ont_1 = \{ \conc{Turtle} \isa \conc{Mammal}, \conc{Mammal} \isa \conc{Animal}, \conc{Turtle} \isa \forall \conc{hasChild}.\conc{Turtle}, \conc{Mammal} \isa \forall \conc{hasChild}.\conc{Mammal} \}$ satisfies $g$, while the theory $\ont_2 = \{ \conc{Turtle} \isa \conc{Mammal}, \conc{Mammal} \isa \conc{Animal} \}$ does not. 
\end{example}

A set $G$ of generators is called a \emph{GBox}. Furthermore, we define
the set $B(G)$ (resp. $H(G)$) as the set of all bodies (resp. heads)
occurring in $G$, i.e., they are sets of ontologies. % by abusing the notation introduced in Def.~\ref{generator}.

\begin{definition}\label{def:expansion}
  Let $G$ be a GBox, $\ont$ an ontology, and $\lang$ a
  language. The \emph{expansion of $\ont$ and $G$ in $\lang$}, written
  $\expand(G, \ont, \lang)$, is the smallest set of theories $\ont'$
  such that 
  \begin{itemize}
  \item[(1)] $\ont' \models \ont$,
  \item[(2)] $\ont'$ satisfies every $g \in G$ w.r.t.~$\lang$, and
  \item[(3)] $\ont'$ is entailment-minimal, i.e. there is no $\ont''$ strictly weaker than $\ont'$ satisfying (1) and (2).
  \end{itemize}
\end{definition}
We call the theories in $\expand(G, \ont, \lang)$ \emph{expansions}.
This definition corresponds to the model-theoretic Datalog semantics, with
consequence rather than set inclusion. Since axioms can be
rewritten to be subset-incomparable, entailment-minimality is
used rather than subset minimality. For example, consider $\{A \sqsubseteq
B, B\sqsubseteq C\}$ and $\{A \sqsubseteq C\}$: the second one is not
a subset of the first one, but weaker than it.

\begin{example}
  Recall the generator $g$ from \Cref{ex:animal-generator}, and let $G$ be a GBox consisting of $g$ alone. Let $\ont = \{ \conc{Turtle} \isa \conc{Mammal}, \conc{Mammal} \isa \conc{Animal} \}$, and let $\lang$ be the set of all concept names. Then $\{ \conc{Turtle} \isa \conc{Mammal}, \conc{Mammal} \isa \conc{Animal}, \conc{Turtle} \isa \forall \conc{hasChild}.\conc{Turtle}, \conc{Mammal} \isa \forall \conc{hasChild}.\conc{Mammal} \} \in \expand(G, \ont, \lang)$.
\end{example}

\section{Results}

We show that the semantics defined in the previous section coincides
with a fixpoint-based one, investigate the role played by the language
$\lang$, and investigate generators with negated
templates.

%\subsection{Fixpoint semantics}

\begin{theorem}\label{thm-expansion-equiv}
  For every $G$, $\ont$, and $\lang$, we have that any two $\ont_1, \ont_2 \in \expand(G, \ont, \lang)$ are logically equivalent. 
\end{theorem}
\begin{proof}
  Assume for contradiction that this is not the case. Then there exist $\ont_1, \ont_2 \in \expand(G, \ont, \lang)$ such that $\ont_1 \not\models \ont_2 \not\models \ont_1$ because otherwise, one would be strictly weaker than the other, contradicting the definition of $\expand(G, \ont, \lang)$. In particular, there exist $\alpha$ and $\beta$ such that: 
  \begin{align}
      \ont_1 &\models \alpha, &\ont_2 \not\models \alpha \label{thm-expansion-equiv-alpha}\\
      \ont_2 &\models \beta, &\ont_1 \not\models \beta \label{thm-expansion-equiv-beta}
  \end{align}

Now consider the set of axioms $T = \{ \tau \mid \ont_1 \models \tau \land \ont_2 \models \tau \}$. Since both $\ont_1$ and $\ont_2$ entail $\ont$ and satisfy every $g \in G$, it is clear that so does $T$. However,
  \begin{align}
      T \not\models \ont_1\\
      T \not\models \ont_2
  \end{align}
  due to the entailments $\alpha$ (Eq.~\ref{thm-expansion-equiv-alpha}) and $\beta$ (Eq.~\ref{thm-expansion-equiv-beta}). Hence $T$ is strictly weaker than both $\ont_1$ and $\ont_2$. This contradicts the initial assumption of $\ont_1, \ont_2 \in \expand(G, \ont, \lang)$.
\end{proof}

Hence applying a GBox $G$ to an ontology $\ont$ results in a theory
that is  unique modulo equivalence, but not necessary finite. As a
consequence, we can treat $\expand(G, \ont, \lang)$ as a single theory when convenient.

Our definition of  $\expand(G, \ont, \lang)$ is strictly semantic,
i.e., does not tell us how to identify any $\ont'\in \expand(G, \ont,
\lang)$. In order to do that, we define a 1-step expansion. 

\begin{definition}\label{def:1step-expansion}
The \emph{1-step expansion of $\ont$ and $G$ in $\lang$}, written
$\oexpand(G, \ont, \lang)$, is defined as follows: 
$$\oexpand(G, \ont, \lang) = \ont\cup \bigcup_{T_B\rightarrow T_H \in G}\{ T_H\sigma \mid \sigma \in \eval(T_B, \ont, \lang)\}.$$
\end{definition}

In other words, we add to $\ont$ all instantiated heads of all
generators applicable in $\ont$. Of course, this extension may result in 
other generators with other substitutions becoming applicable, and so on recursively.

\begin{lemma}%[$\oexpand$ monotonicity]
  If $\ont_1 \subseteq \ont_2$, then $\oexpand(G, \ont_1, \lang) \subseteq \oexpand(G, \ont_2, \lang)$.
\begin{proof}
    Simple consequence of Def.~\ref{sec:generator:satisfaction} and $\eval(B(g), \ont_1, \lang) \subseteq \eval(B(g), \ont_2, \lang)$ for any generator $g$.
\end{proof}
\end{lemma}

\begin{definition}
The \emph{$n$-step expansion of $\ont$ and $G$ in $\lang$}, written
$\oexpand^n(G, \ont, \lang)$, is defined as follows: 
$$\oexpand^n(G, \ont, \lang) =
\underbrace{\oexpand(\ldots\oexpand(}_{n \text{times}}G, \ont, \lang)
\dots).$$
We use $\oexpand^*(G, \ont, \lang)$ to denote the least fixpoint of
$\oexpand(G, \ont, \lang)$. 
\end{definition}

\begin{theorem}
  \label{expansion-fixpoint}
  For finite $\lang$, the least fixpoint $\oexpand^*(G, \ont, \lang)$ exists and belongs to $\expand(G, \ont, \lang)$.
\end{theorem}
\begin{proof}
  Since $\lang$ is finite, the set of all $\lang$-substitutions for the variables occurring in $G$ is finite. Let $\Sigma_{\lang}$ be this set, and consider the set $H = \ont \cup \displaystyle\bigcup_{T_B \rightarrow T_H \in G, \sigma \in \Sigma_{\lang}} T_H\sigma$, that is, $\ont$ as well as all axioms obtained from the heads of instances of generators in $G$. This set is also finite. 

It is easily verified that $\oexpand$ is an operator on the powerset of $H$. Since $\oexpand$ is monotone, the least fixpoint $\oexpand^*(G, \ont, \lang)$ exists, and belongs to $\expand(G, \ont, \lang)$ by construction.

% To show that $\oexpand^*(M, \ont, \lang) \in \expand(M, \ont, \lang)$, it suffices to observe that if this is not the case, then for some $T_B \rightarrow T_H \in M$ and $\sigma$, $\oexpand^*(M, \ont, \lang) \models T_B\sigma$ and $\oexpand^*(M, \ont, \lang) \not\models T_H\sigma$. In this case we have that $T_H\sigma \not\subseteq \oexpand^*(M, \ont, \lang)$. However, we also have that $T_H\sigma \subseteq \oexpand(M, \oexpand^*(M, \ont, \lang), \lang)$, contradictng the fact that $\oexpand^*(M, \ont, \lang)$ is a fixpoint.
\end{proof}

In other words, our  fully semantic definition of
$\expand(G,\ont,\lang)$ coincides with 
the operational semantics based on the fixpoint computation.

\paragraph{Size of the fixpoint} For a generator $g$ with variables
$V$, there are at most $|\lang|^{|V|}$ different
$\lang$-substitutions. The size of the fixpoint is therefore bounded
by $|G| \times |\lang|^{n}$, where $n$ is the maximum number of
variables in any $g \in G$. In the worst case we need to perform
entailment checks for all of them, adding one instantiation at a time
to $\ont$. 
Hence determining  $\oexpand^*(G, \ont, \lang)$ involves up to $(|G| \times |\lang|^{n})^2$ entailment
checks. 
For finite $\lang$ and provided we have a fixed upper bound for $n$,
 determining  $\oexpand^*(G, \ont, \lang)$ involves a polynomial
 number of entailment tests and results
 in a $\oexpand^*(G,
\ont, \lang)$ whose size is  polynomial in the size of $G$ and $\lang$ . 

\subsubsection*{Finite vs infinite L}

 The next examples illustrate the difficulties an infinite
  language $\lang$ can cause. The first example shows how an infinite
  $\lang$ can lead to infinite expansions.  

\begin{example}
 Consider the ontology $\ont= \{ \conc{A}\sqsubseteq \exists \conc{R}.\conc{B}\}$, the
 generator 
 $g:\{?X \sqsubseteq \exists \conc{R}.?Y\}\rightarrow 
 \{ ?X \sqsubseteq \exists \conc{R}.\exists \conc{R}.?Y\}$, and $\lang$ the set of all
 $\mathcal{EL}$-concept expressions. Clearly, 
 $\oexpand^*(G, \ont, \lang)$ is infinite, and so is each expansion in $\expand(G, \ont, \lang)$.
\end{example}

The next example shows that this does not necessarily happen. 

\begin{example}
 Consider the ontology $\ont= \{\exists \conc{R}.\conc{A}\sqsubseteq \conc{A}\}$, the
 generator 
 $g:
 \{\exists \conc{R}.?X \sqsubseteq ?X\}\rightarrow
\{\exists \conc{R}.\exists \conc{R}.?X \sqsubseteq \exists \conc{R}.?X \}
$, and $\lang$ the set of all
 $\mathcal{EL}$-concept expressions. Clearly, 
 $\oexpand^*(G, \ont, \lang)$ is infinite, but there is a finite (and
 equivalent) ontology to this fixpoint in $\expand(G, \ont, \lang)$,
 namely  $\ont$ itself. 
\end{example}

While having to explicitly specify $\lang$ may seem to be cumbersome, it is not very restrictive. In fact, it is easy to show that, for finite languages, generators can be rewritten to account for concepts, roles, or individuals that are missing from a given language by grounding the generators. 

\begin{definition}%[$\lang$-grounding]
  Let $g : T_B \rightarrow T_H$ be a generator, and $\lang$ a finite language. The \emph{$\lang$-grounding of $g$} is the finite set of generators $\{T_B\sigma \rightarrow T_H\sigma \mid \sigma \mbox{ an } \lang\mbox{-substitution}\}$. 
\end{definition}

Using $\lang$-grounding, we can compensate for a smaller language
$\lang_1\subsetneq \lang_2$ by  $\lang_2\setminus \lang_1$-grounding
generators, thereby proving the following theorem. 

\begin{theorem}
  Let $\lang_1 \subseteq \lang_2$ be finite languages. For every GBox $G$
  there exists a Gbox $G'$ such that, for every $\ont$, $\ont_1$, $\ont_2$ we have that $\ont_1\in \expand(G', \ont, \lang_1)\text{ and } 
\ont_2\in
\expand(G, \ont, \lang_2)\text{ implies }  \ont_1 \equiv \ont_2. $
\end{theorem}
\begin{proof}
  Take $G'$ to be the union of the $\lang_2$-groundings of every generator in $G$. 
\end{proof}

Of course, grounding all the generators is a very wasteful way of
accounting for a less expressive language. A more clever rewriting
algorithm should be possible: for example, if we allow binary conjunctions of
names in $\lang_2$ but not in $\lang_1$, we can add copies
of each generator where we replace variables $?X$ with $?X_1 \sqcap
?X_2$.

\subsection{GBox containment and equivalence}
\label{sec:containment}

Having defined GBoxes, we now investigate a suitable notion for containment and equivalence of GBoxes. 

\begin{definition}[$\lang$-containment]
  Let $G_1$ and $G_2$ be GBoxes, and $\lang$ a language. $G_1$ is
  \emph{$\lang$-contained in} $G_2$ (written $G_1 \preceq_\lang G_2$) if $\expand(G_2, \ont, \lang) \models \expand(G_1, \ont, \lang)$ for every ontology $\ont$.
\end{definition}

The following lemma relating the entailment of theories and the
entailment of expansions holds as a direct consequence of the monotonicty
of description logics.

\begin{lemma}\label{lem:twentynine}
Let $G$ be a GBox, $T,T'$ two theories and $\lang$ a language. If
$T\models T'$ then $\expand(G,T,\lang)\models \expand(G,T',\lang)$.
\end{lemma}

Furthermore, the following is a rather straightforward
consequence of the definition of the semantics of generators. 

\begin{lemma}
  \label{lemma:exp-entailment}
Let $T$ be a theory, $G$ a GBox, $\ont$ an ontology, and $\lang$ a
language. If $T \models \ont$ and $T$ satisfies every generator $g \in G$
then $T \models \expand(G, \ont, \lang)$.  
\end{lemma}

Using Lemmas~\ref{lem:twentynine} and \ref{lemma:exp-entailment},
$\lang$-containment can be shown to be decidable, and in fact efficiently so, using a standard freeze technique from database theory.
\begin{theorem}
Let $G_1$ and $G_2$ be GBoxes, and $\lang$ a language. $G_1$ is
$\lang$-contained in $G_2$ if and only if $\expand(G_2, T_B,\lang) \models T_H$ for
every $T_B\rightarrow T_H \in G_1$.
\end{theorem}
\begin{proof}
  The only-if direction follows directly. For the other direction, by
  \cref{lemma:exp-entailment} we need to show  that if  $\expand(G_2, T_B,\lang)
  \models T_H$ for all $T_B\rightarrow T_H\in G_1$ then for any
  ontology $\ont$
\begin{align}
 &\expand(G_2, \ont,\lang) \models g \text{ for all }g \in G_1, \label{cond:1}\\
 &\expand(G_2, \ont,\lang) \models \ont.\label{cond:2}
\end{align}
By \cref{lemma:exp-entailment}, \eqref{cond:1} and \eqref{cond:2} imply
$\expand(G_2,\ont,\lang) \models \expand(G_1,\ont,\lang)$, which is the definition of $G_1$ being $\lang$-contained in $G_2$.
\eqref{cond:2} is an immediate consequence of the definition of the expansion, hence we only need to show \eqref{cond:1}.

In the following we slightly abuse notation: $\expand(G,
\ont,\lang)$ for a GBox $G$, ontology $\ont$ and language $\lang$
shall refer to an ontology as opposed to a set of possible expansions;  by \cref{thm-expansion-equiv}, they are all logically
equivalent.

Let $T_B\rightarrow T_H\in G_1$ be fixed but arbitrary. Furthermore, let $\sigma \in \eval(T_B, \expand(G_2,\ont,\lang))$.

 Then, by the definition of $\eval$, 
\begin{align}
\expand(G_2, \ont,\lang) \models T_B\sigma. \tag{*}\label{cond:star}
\end{align} 
Applying \cref{lem:twentynine} to \eqref{cond:star} yields
 $\expand(G_2, \expand(G_2, \ont,\lang),\lang) \models \expand(G_2,
 T_B\sigma,\lang)$. But $\expand(G_2, \expand(G_2,\ont,\lang),\lang) =
 \expand(G_2,\ont,\lang)$ (otherwise $\expand(G_2,\ont,\lang)$ would not be an expansion) and hence

\begin{align}
\expand(G_2,\ont,\lang) \models \expand(G_2, T_B\sigma,\lang).\label{cond:4}
\end{align} 
Thus what remains is to show that 

\begin{align}
 \expand(G_2, T_B\sigma,\lang) \models T_H\sigma,\label{cond:5}
\end{align}
 since \eqref{cond:4} and \eqref{cond:5} together yield 

\begin{align}
\expand(G_2,\ont,\lang)\models T_H\sigma. \tag{**}\label{cond:starstar}
\end{align} 
  which together with \eqref{cond:star} implies that
  $\expand(G_2,\ont,\lang)$ satisfies $T_B\rightarrow T_H$. 

 Using compositionality of $\lang$-substitutions and the iterative fixpoint
 construction of the expansion, it is straightforward to show that 
\begin{align}
 \expand(G_2, T_B\sigma,\lang) \models \expand(G_2, T_B,\lang)\sigma.\label{cond:7}
\end{align}

 By the assumption of the theorem, $\expand(G_2, T_B,\lang)\models
 T_H$ which in turn implies that $\expand(G_2, T_B,\lang)\sigma
 \models T_H\sigma$. This together with \eqref{cond:7} yields 

\begin{align}
\expand(G_2, T_B\sigma,\lang) \models (\expand(G_2, T_B,\lang)\sigma \models T_H\sigma,
\end{align} 
thus proving \eqref{cond:5} and thereby \eqref{cond:starstar}, as desired.
\end{proof}

It follows that $\lang$-containment is decidable for arbitrary $\lang$
 (even infinite), since we can restrict ourselves to the language of
 all subexpressions of $B(G_1)$. Furthermore,
the complexity is the same as that of computing an expansion of a GBox.

\subsection{GBoxes with negation}\label{sec:negation}

In this section we introduce negation-as-failure to GBoxes. We extend the definition of the
expansions defined in Section~\ref{sec:syntax-semantics}, define suitable notions of semi-positive GBoxes and
semantics for stratified GBoxes, and prove the corresponding
uniqueness results.

To do so, a generator is now a rule of
the form $T_B^+(V_1), \naf T_B^-(V_2)\rightarrow T_H(V_3)$, for
$T_B^+(V_1),  T_B^-(V_2), T_H(V_3)$ templates 
with $V_3 \subseteq V_1\cup V_2$. For the sake of notational simplicity,
we restrict ourselves here to generators with at most one template in the negative body. It is worth noting, however, that all
definitions and results in this section are immediately transferable
to generators with multiple templates in the negative
bodies (multiple templates in the positive body can of course be
simply merged into a single template).

The following definition, together with Definition~\ref{def:expansion}
of $\expand(G,\ont,\lang)$, provides a minimal model semantics for
GBoxes with negation:
\begin{definition}
 An ontology $\ont$ \emph{satisfies a generator $g:T_B^+(V_1), \naf T_B^-(V_2)\rightarrow T_H(V_3)$  wrt. }$\lang$ if, for every $\sigma \in \eval(T_B^+, \ont, \lang) \setminus \eval(T_B^-, \ont, \lang)$ we have $\ont \models T_H\sigma$.
\end{definition}

Unsurprisingly, adding negation
results in the loss of uniqueness of the expansion
$\expand(G,\ont,\lang)$ (cf. Theorem~\ref{thm-expansion-equiv}), as
illustrated by the following example. 
\begin{example}\label{ex:nonunique-stratified}
  Let $\lang=\{\conc{A}, \conc{B}, \conc{C}, \conc{s}\}$,
  $\ont=\{\conc{A}(\conc{s})\}$ and $G=\{\conc{A}(?X), \naf \conc{B}(?X)\rightarrow \conc{C}(?X) \}$.
  Then $\expand(G,\ont,\lang)$ contains the two non-equivalent
  expansions $\{\conc{A}(\conc{s}), \conc{B}(\conc{s})\}$ and $\{\conc{A}(\conc{s}), \conc{C}(\conc{s})\}$.
  % Consider the language $\lang= \{\conc{A}, \conc{B}\}$, the empty
  % ontology  $\ont$ and a GBox containing the two
  % generators $\{?X\isa \conc{B}\} \rightarrow \{?X \isa \conc{B}\}$
  % and $\naf \{?Y \isa \conc{B}\} \rightarrow \{?Y\isa \neg \conc{B}\}.$
  % Then $\expand(G,\ont,\lang)$ contains
  % the two non-equivalent expansions
  % $\{\conc{A}\isa \conc{B}\}$ and
  % $\{\conc{A}\isa \neg \conc{B}\}$.
\end{example}
Next, we extend the definition of the 1-step expansion operator
from Definition~\ref{def:1step-expansion} to support
negation. However, as Example~\ref{ex:fixpoint-nomodel} will show, a
fixpoint does not always correspond to an expansion in
$\expand(G,\ont,\lang)$.

\begin{definition}\label{def:inflationary-operator}
The \emph{1-step expansion of $\ont$ and $G$ in $\lang$} of a GBox $G$ with negation, written
$\oexpand^-(G, \ont, \lang)$, is defined as follows: 
$$\oexpand^-(G, \ont, \lang) = \ont\cup \bigcup_{T_B^+, \naf
  T_B^-\rightarrow T_H \in G}\{ T_H\sigma \mid \sigma \in \eval(T_B^+,
\ont, \lang)\setminus \eval(T_B^-, \ont, \lang)\}.$$
\end{definition}

\begin{example}\label{ex:fixpoint-nomodel}
Consider the ontology $\ont=\{\conc{Single}\sqsubseteq \conc{Person},
\conc{Spouse}\sqsubseteq \conc{Person}, \conc{Single}\sqsubseteq \neg \conc{Spouse},
\conc{Person}(\conc{Maggy})\}$ and the following GBox $G$ 
$$\begin{array}{rrcl}
G=\{&\{\conc{Person}(?X)\},\naf \{\conc{Single}(?X)\}&\rightarrow \{\conc{Spouse}(?X)\},\\
   &\{\conc{Person}(?X)\}, \naf \{\conc{Spouse}(?X)\}&\rightarrow \{\conc{Single}(?X)\}\}
\end{array}$$
The expansion $\expand(G, \ont, \lang)$ contains the two non-equivalent
ontologies $\ont \cup \{\conc{Single}(\conc{Maggy})\}$ and $\ont\cup
\{\conc{Spouse}(\conc{Maggy})\}$. Furthermore, the iterated fixpoint 
$(\oexpand^-)^*(G,\ont,\lang)$ is $\ont \cup \{\conc{Single}(\conc{Maggy}),
\conc{Spouse}(\conc{Maggy})\}$; this is, however, not an ontology in
$\expand(G,\ont,\lang)$ as it is not entailment-minimal.
\end{example}

A natural question arising is whether we can identify or even
characterize GBoxes with negation that have a unique
expansion. To this end, we
define suitable notions of semi-positive GBoxes and stratified
negation. These are based on the notion of multiple templates affecting
others, as formalized next.
\begin{definition}
  Let $\lang$ be a language, $S=\{S_1,\ldots,S_k\}$ a set of
 templates, $\ont$ an ontology, and $T$ a template. We
 say that \emph{$S$ activates $T$ with respect to $\ont$ and $\lang$}  if there exist
$\lang$-substitutions $\sigma_1,\ldots\sigma_k$ such that $\ont\cup \bigcup
S_i\sigma_i\models T\sigma$ for some $\lang$-substitution
$\sigma$. For brevity we omit $\ont$ and $\lang$ if they are clear
from the context.
\end{definition}
In contrast to standard Datalog with negation, the entailment of a
template in the body of a generator is not solely dependent on a single
generator with a corresponding head firing. Instead, multiple generators might
need to fire and interact with $\ont$ in order to entail a body template. Hence we use the set
$S$ of templates in the definition of activation.

\begin{example}
  Consider the GBox containing $g_1\colon T_1(?X)\rightarrow
  \{?X\sqsubseteq \conc{A}\}, g_2\colon T_2(?Y)\rightarrow \{?Y\sqsubseteq
  \conc{B}\}$ and $g_3\colon \naf \{?Z\sqsubseteq \conc{A} \sqcap \conc{B}\}\rightarrow T_3(?Z)$.
 Then $H(g_1)$ and $H(g_2)$ activate $\{?Z\sqsubseteq
 \conc{A}\sqcap \conc{B}\}$ with respect to any $\ont$ and $\lang$,
 indicating that the firing of $g_3$ depends on the combined firing of
 $g_1$ and $g_2$.
\end{example}

Activation can then be used to define a notion of semi-positive
GBoxes, which is analogous to semi-positive Datalog programs.
\begin{definition}[Semi-positive GBoxes]
Let $G$ be a GBox with negation, $\lang$ a
language, and $\ont$ an ontology. $G$ is called \emph{semi-positive
  w.r.t. $\ont$ and $\lang$} if no negative body template
$T_B^-$ of a generator $g\in G$ is activated by
$H(G)$.
\end{definition}

As seen in \cref{ex:nonunique-stratified}, even semi-positive GBoxes 
result in multiple non-equivalent expansions. In that example, neither
the ontology $\ont$ nor any possible firing of $G$ can yield
$B(s)$. As such, we wish to restrict the theories in
$\expand(G,\ont,\lang)$ to containing only facts derivable from $\ont$
and $G$. To that end, the following definition
suitably restricts the entailment of expansions. 

\begin{definition}\label{def:justifiable}
Let $G$ be a GBox, $\ont$ an ontology, and $\lang$ a finite
language. We say that an expansion $\ont'\in \expand(G,\ont,\lang)$ is
\emph{justifiable w.r.t. $(G,\ont,\lang)$} if the
following holds: if $\ont'\models T\sigma$ for some template $T$ and
substitution $\sigma$, then $\ont\models T\sigma$ or $H(G)$ activates
$T\sigma$ with respect to $\ont$ and $\lang$. We write simply $\ont'$
is \emph{justifiable} when $G$, $\ont$, and $\lang$ are clear from the context.
\end{definition}
Using this notion, we can show that, indeed,a GBox being semi-positive
implies that its semantics is unambiguous when restricted to
justifiable expansions.  

\begin{theorem}\label{thm:semipositive}
Let $G$ be a semi-positive GBox, $\ont$ an ontology, and $\lang$ a
finite language. Then the fixpoint $(\oexpand^-)^*(G,\ont,\lang)$ exists,
is the unique fixpoint of $\oexpand^-$, and is contained in $\expand(G,\ont,\lang)$.
\end{theorem}
\begin{proof}
  Since $\oexpand^-(G,\ont,\lang)$ is an inflationary operator and $L$
  is finite, there exists an iterative fixpoint
  $O^*=(\oexpand^-)^*(G,\ont,\lang)$. By
  construction, $\ont^*$ satisfies $\ont$ and all generators $g\in
  G$ and is justifiable w.r.t. $(G,\ont,\lang)$. We simultaneously prove uniqueness and membership in
  $\expand(G,\ont,\lang)$ by showing that $O'\models O^*$
  for an arbitrary justifiable expansion $O'\in\expand(G,\ont,\lang)$.
   Let $\ont_0=\ont$ and $\ont_i=\oexpand^-(G,\ont_{i-1},\lang)$ for
   $i\geq 1$, then $\ont=\ont_0\subseteq \ldots \subseteq \ont_k=\ont^*$ for
   some $k$. Assume $\ont_1\models T\sigma$ for some
   $\lang$-substitution $\sigma$ and $T\in H(G)$. Then either $\ont\models T\sigma$ (in which case $\ont'\models T\sigma$) or there exists
   a generator $$ T_B^+,\naf T_B^-\rightarrow T$$ such that $\sigma\in
   \eval(T_B^+,\ont, \lang)$ and $\sigma\not\in \eval(T_B^-,\ont,
   \lang)$. Since $G$ is semi-positive, $H(G)$ cannot activate
   $T_B^-$, i.e., there exists no set of generators that, together
   with the ontology $\ont$, could fire in
   a way that would entail $T_B^-\sigma$. Since $\ont'$ is
   entailment-minimal and justifiable, it must be the case that $\ont'\not\models
   T_B^-\sigma$ and hence $\ont'\models T\sigma$. Thus, $\ont'\models
   \ont_1$.

   The same argument can be applied inductively to show that
   $\ont'\models \ont_i$ for $i\geq 1$, thus showing $\ont'\models
   \ont^*$. Since $\ont'$ was chosen arbitrarily, this proves both
   the uniqueness and membership claims.
 \end{proof}
The following is a direct corollary of the proof of Theorem~\ref{thm:semipositive}.
\begin{corollary}
Let $G$ be a semi-positive GBox, $\ont$ and ontology and $\lang$ a
finite language. All justifiable ontologies in $\expand(G,\ont,\lang)$ are
logically equivalent.
\end{corollary}

For a GBox to be  semi-positive is a very strong requirement. Next, we
introduce the notion of a stratified GBox: this does not ensure that
all expansions are equivalent, but it ensures that we can determine
one of its expansions by expanding strata in the right order.  Again,
we use $H(G)$ to denote the set of templates in heads of generators
in $G$, and $B(G)$ for the set of templates in (positive or
negative) bodies of generators in $G$. 

\begin{definition}[Stratification]\label{def:stratification}
Let $\lang$ be a language and $\ont$ an ontology. A GBox  $G$ 
is \emph{stratifiable w.r.t. $\ont$ and $\lang$} if there exists a
function $v: H(G)\cup B(G) \rightarrow \mathbb N$ such that, for every
generator $T_B^+, \naf T_B^-\rightarrow T_H\in G$ the following
 holds:

\begin{enumerate}
\item $v(T_H)\geq v(T_B^+)$,  
\item $v(T_H)> v(T_B^-)$, 
\item  for every $\subseteq$-minimal  $S_1 \subseteq H(G)$ that activates $T_B^+$,
$v(T^+_B)\geq \max\limits_{S'\in S_1} v(S')$,
\item for every $\subseteq$-minimal  $S_2 \subseteq H(G)$ that activates $T_B^-$,
$v(T^-_B)> \max\limits_{S'\in S_2} v(S')$.
\end{enumerate}
% \begin{enumerate}
% \item $v(\alpha)\leq v(\alpha_H)$ for every axiom $\alpha \in T_B^+$
%   and every $\alpha_H\in T_H$,
% \item  $v(\alpha^-)< v(\alpha_H)$ for every axiom $\alpha^- T_B^-$
%  and every $\alpha_H\in T_H$, and 
% \item for every $\subseteq$-minimal
%   $S_1,S_2\subseteq H(G)$ that activate $T^+$ or $T^-$, respectively,
%   we have that 
% %
% $$v(\alpha)\geq \max\limits_{\beta\in S_1} v(\beta)\text{ and }
% v(\alpha^-)> \max\limits_{\beta\in S_2} v(\beta)$$ 
% %
% for every  axiom $\alpha \in T_B^+$, $\alpha^- T_B^-$. 
% \end{enumerate}
\end{definition}

The first two conditions in the previous definition are analogous to
stratified Datalog, which intuitively states that a body literal must be
evaluated (strictly, in the case of negative literals) before head
literals.
The second two conditions tailor the stratification to generators: generators allow
for more interaction amongst their components. As opposed to Datalog,
multiple heads combined might be needed to entail a body
template. Thus, a body template must be defined in a higher stratum
than any possible set of templates that could entail it.

Following this definition, a stratification $v$ of a GBox $G$
w.r.t. an ontology $\ont$ gives rise to a partition $G_v^1,\ldots G_v^k$
of $G$, where each generator $g:T_B^+, \naf T_B^-\rightarrow
T_H$ is in the stratum $G_v^{v(T_H)}$.

For a GBox $G$, an ontology $\ont$ and a language $\lang$, we can define the \emph{precedence graph} $\mathcal
G_{G, \ont,\lang}$ as follows: nodes are the templates occuring in $G$ and
\begin{enumerate}
\item if $T_B^+,\naf T_B^-\rightarrow T_H$ is in $G$, then $\mathcal
  G_{G,\ont,\lang}$ contains the positive edge $(T_B^+,T_H)$ and the negative edge
  $(T_B^-,T_H)$;
\item for a template $T$ that occurs in the positive (resp. negative) body of a
  generator and any $\subseteq$-minimal set
  $\{S_1,\ldots,S_k\}\subseteq H(G)$ that activates $T$ w.r.t. $\ont$
  and $\lang$, $\mathcal G_{G,\ont,\lang}$ contains the positive (resp. negative) edges $(S_i,T)$ for $1\leq i
  \leq k$.
 \end{enumerate}
% nodes are templates occurring in $G$ and, if $T_B^+,\naf T_B^-\rightarrow
% T_H\in G$, and 
% \begin{enumerate}
% \item  if $\alpha^+\in
%   T_B^+$ and $\alpha\in T_H$, then $\mathcal
%   G_G$ contains the positive edge $(\alpha^+,\alpha)$, 
% %
% \item  if  $\alpha^-\in
%   T_B^-$ and $\alpha\in T_H$, then $\mathcal
%   G_G$ contains the negative edge $(\alpha^-,\alpha)$, 
% %
% \item for an axiom $alpha$ that occurs in the positive
%   (resp. negative) body $T$ of a
%   generator and any $\subseteq$-minimal set
%   $\{S_1,\ldots,S_k\}\subseteq H(G)$ that activates $T$,
%   $\mathcal G_{G,\ont,\lang}$ contains the positive (resp. negative) edges $(S_i,T)$ for $1\leq i
%   \leq k$.
% \end{enumerate}
We then get the following classification of stratified GBoxes, the
proof of which is entirely analogous to the Datalog case.

\begin{proposition}
Let $\lang$ be a language and $\ont$ an ontology. A GBox $G$ is
stratifiable w.r.t. $\ont$ and $\lang$
iff its precedence graph $\mathcal G_{G,\ont,\lang}$ has no cycle
with a negative edge.
\end{proposition}
Given such a
stratification, we can thus define a semantics for stratified negation.

\begin{definition}[Stratified semantics] 
Let $\ont$ be an ontology, $\lang$ a language, and  $G$ a GBox
stratifiable w.r.t. $\ont$ and $\lang$. For a
stratification $v$ of $G$ and the induced partition $G_v^1,\ldots,
G_v^k$ of $G$, we define $\ont_v^{\mathsf{strat}}(G,\ont,\lang)$ as follows:
\begin{enumerate}
\item $\ont_v^1=\ont$,
\item $\ont_v^j=\oexpand^*(G^{j-1},O^{j-1},\lang)$ for $1<j\leq k$,
\item $\ont_v^{\mathsf{strat}}(G,\ont,\lang)=O_v^k$.
\end{enumerate}
\end{definition}
\begin{theorem}
Let $\ont$ be an ontology, $\lang$ a finite language, and $G$ be a GBox
stratifiable w.r.t. $\ont$ and $\lang$. Then $\ont_v^{\mathsf{strat}}(G,\ont,\lang)$ exists, is independent of the choice of $v$, and contained in $\expand(G,\ont,\lang)$.
\end{theorem}
\begin{proof}

Let $G_v^1,\ldots, G_v^k$ be the partitioning of $G$ w.r.t. to a
stratification $v$. By Definition~\ref{def:stratification}, each
$G_v^i$ is a semi-positive GBox. Hence Theorem~\ref{thm:semipositive}
guarantees the existence of
$\ont_v^{\mathsf{strat}}(G,\ont,\lang)$. By construction
$\ont_v^{\mathsf{strat}}(G,\ont,\lang)$ satisfies $\ont$ and all
generators in $G$. Furthermore, there cannot exist an ontology $\ont'$
such that $\ont_v^{\mathsf{strat}}(G,\ont,\lang)\models \ont'$
satisfying $\ont$ and all generators in $G$, as this would contradict
the entailment-minimality of the $\ont_v^i$.

The proof for the independence of the stratification $v$ is entirely
analogous to the Datalog case: the strongly connected components of
$\mathcal G_{G,\ont,\lang}$ provide the most granular stratification, which can
then be used to prove the equivalence of all stratifications
(cf. \cite{DBLP:books/aw/AbiteboulHV95} for a proof for stratified Datalog).
\end{proof}

\begin{remark}\label{rem:unique-nonstratified}
It is worth noting that, although the stratified semantics provides a
unique model, stratified GBoxes do not necessarily have a unique
expansion. For example, the GBox from
Example~\ref{ex:nonunique-stratified} is
stratifiable yet has multiple distinct expansions. Moreover, just as in
Datalog, there exist nonstratified GBoxes that have a unique
expansion.% : let $S,T$ be templates and $\ont=\{T(A,B)\}$ with the GBox $\{T(?X,?Y), \naf S(?Y)\rightarrow S(?X)\}$. This is obviously
% not stratifiable as no stratum can be assigned to $S$. However, a
% unique fixpoint exists, namely $\{T(a,b), S(A)\}$, and all expansions 
% in $\expand(G,\ont,\lang)$ must by minimality be equivalent to it.
\end{remark}
% \begin{definition}[Stratification]
%   Let $M$ be a set of mappings, $\ont$ an ontology, and $\lang$ a language. A stratification of $M$ with respect to $\ont$ and $\lang$ is a partition of $M$ into sets $M_1,\ldots,M_n$ such that for all $i<n$, $\expand(O, \bigcup_{k=1}^i M_k, L) \not\models T_B$.
% \end{definition}

%%% Local Variables:
%%% mode: latex
%%% TeX-master: "main-techrep"
%%% End:

%%% Local Variables:
%%% mode: latex
%%% TeX-master: "main-techrep"
%%% End:

%\input{open-problems}

\section{Related work}
When combining rules with DL ontologies, the focus has thus far primarily been on (1) encoding ontology axioms in rules for efficient query answering and (2) expanding the expressivity of ontologies using rules. In contrast, GBoxes are designed as a tool for ontology specification by describing instantiation dependencies between templates.

\textbf{Datalog$^{\pm}$} \cite{DBLP:conf/datalog/CaliGLP10} falls into
the first category: it provides a formalism for unifying ontologies
and relational structures. Datalog$^{\pm}$ captures ontology
axioms as rules, and these cannot ``add'' new axioms. 

\textbf{dl-programs} \cite{EITER20081495} and \textbf{DL-safe
  rules} \cite{MOTIK200541} fall into the second category: dl-programs add nonmonotonic reasoning by means of stable model semantics, whereas DL-safe rules allow for axiom-like rules not expressible in standard DL. However, none of these formalisms adds new TBox axioms to the ontology.

\textbf{Tawny-OWL}\footnote{\url{https://github.com/phillord/tawny-owl}} and the \textbf{Ontology Pre-Processing Language}\footnote{\url{http://oppl2.sourceforge.net/index.html}} (OPPL) are formalism for manipulating OWL ontologies \cite{DBLP:conf/owled/Lord13,DBLP:conf/owled/EganaSA08}.  While OPPL was designed to capture patterns and regularities in ontologies, Tawny-OWL is a more general programmatic environment for authoring ontologies that includes powerful  support for ontology design patterns. It is part of future work to see whether GBoxes can be faithfully implemented in Tawny-OWL (OPPL lacks the recursion required).

Another question is whether \textbf{metamodeling} in DL, in particular the encoding scheme from \cite{DBLP:conf/semweb/GlimmRV10} can be faithfully captured by (an extension of) GBoxes: this would require \emph{replacing} axioms in $\ont$ with others which is currently not supported.

 % describes the idea of considering concepts and roles as individuals pertaining to metaconcepts. Such metaconcepts allow to reason over categories of concepts and roles, and relate them via metaroles. One proposal to enable metamodeling in OWL describes an encoding scheme within OWL 2 without introducing a language extension \cite{DBLP:conf/semweb/GlimmRV10}. While this encoding scheme specifies translation rules for the axioms of an ontology into two sets of axioms representing the instance layer and the metalayer, these rules cannot be captured by GBoxes in a straightforward manner. Gboxes do not allow to replace entities or duplicate axioms such that all their entities are renamed in a consistent manner.

%\subsection{Ontology Design Patterns}
\label{ODP}
\textbf{Ontology Design Patterns} (ODPs) have been proposed to capture best practices for developing ontologies \cite{DBLP:conf/semweb/Gangemi05,DBLP:conf/iceis/BlomqvistS05},  inspired by Software Design Patterns. While some ODPs are easily expressible in GBoxes,  it is part of ongoing work to investigate extensions required to capture others. 

% , ODPs are supposed to provide reusable solutions to common ontology design tasks that may guide ontology engineers during ontology development. It is hypothesised that ODPs have a potential to foster ontology reuse, improve maintainability, and ease ontology comprehension. Gboxes can provide a well-defined framework for using ODPs. Lifting the level of abstraction such that ODPs may be manipulated as language primitives, Gboxes can not only express relationships and dependencies between ODPs but can automatically enforce such interrelationships. %EXAMPLES (make a good argument with respect to dependencies/constrains, compositionality, ..)

\textbf{Reasonable Ontology Templates}\footnote{\url{http://ottr.xyz}}
(OTTR) \cite{OTTR-ISWC18,dl-templates} provide a framework for
macros in OWL ontologies, based on the notion of templates. In contrast to GBoxes, ``matching'' of templates is defined syntacically and non-recursively, but they can be named and composed to give rise to more complex templates. 

% Since OTTR is purely syntactic, combining Gboxes with OTTR promises to expand on this by
% partially automating template instantiation and adding
% a semantic component.

The \textbf{Generalized Distributed Ontology, Modelling and
  Specification Language} (GDOL) \cite{KriegBrckner2017GenericOA} is a
formalism facilitating the template-based construction of ontologies
from a wide range of logics. In addition to
concepts, roles, and individuals, parameters may be ontologies
which act as preconditions for template instantiation: for a given
substitution, the resulting parameter ontology must be satisfiable in
order to instantiate the template. Thus these preconditions serve only
as a means to restrict the set of allowed instantiations of a
template, whereas in GBoxes, an ontology triggers 
such substitutions.

%%%OLD TEXT BELOW

\section{Future work}

We have presented first results about a template-based language for capturing recurring ontology patterns and using these to specify larger ontologies. Here, we list some areas that we would like to investigate in the future.

\paragraph{Finite representability} In general, the semantics of GBoxes is such that the expansion of a GBox and ontology can  be infinite if the substitution range given by $\lang$ is infinite. A natural question arising is whether/which other mechanisms can ensure that \emph{some} expansion is finite, and how can we compute such a finite expansion? 
%
%To control this, we currently introduce a finite language as a separate concept. This %problem, however, raises the following interesting questions:
%
%\begin{itemize}
  %\item % If $\lang$ is finite, so is $\oexpand^*(G, \ont, \lang)$ for any $G$ and $\ont$.
 % Firstly, are
 % there ways,  other than finite $\lang$, to ensure that $\expand(G, \ont, \lang)$ is
 % finitely representable, i.e., contains a (finite) ontology?
%\item
%
%I changed the wording of the next sentence because $G,\ont,\lang$ *is* such a finite representation!
Furthermore, given $G,\ont,\lang$, when can we decide whether an ontology in $\expand(G, \ont, \lang)$ is finite?
%\end{itemize}

% More broadly, the idea behind GBoxes is to capture and represent repeating patterns. To this end, we wish to define a suitable notion of ``nontrivial consequences'' that a template should capture and that a generator should generate new axioms from. 

\paragraph{Controlling substitutions} So far, we have only considered 
 entailment for generators when determining matching substitutions.  Consider the ontology $\ont=\{\conc{A}
\sqsubseteq \conc{B}, \conc{B}\sqsubseteq \conc{C}\}$ and the template
$?X\sqsubseteq C$. The resulting substitutions include  concepts $\conc{A}$  and $\conc{B}$, but also 
a multitude of possibly unwanted, redundant concepts, e.g.,
$\{\conc{A}\sqcap \conc{A}, \conc{A}\sqcap \conc{B},
 \ldots\}$. 
% Restricting query evaluation to syntactic entailment addresses the
% issue of capturing results at the cost of
% ontology reasoning, yielding only $\{\conc{A}\}$ in the above example. 
% %This is clearly highly restrictive and likely not always desirable. 
Hence restricting substitutions to ``reasonable'' or possible ``parametrizable'' (e.g., maximally general) ones is part of future work.  
%Finding suitable entailment relations between these two extremes would be highly %beneficial to GBoxes and is as such an important direction for future work.
% evaluating a template as a query relies on a semantic notion of
% querying. This has the obvious benefit of taking ontology reasoning
% into account while querying. Consider the ontology $\ont=\{\conc{A}
% \sqsubseteq \conc{B}, \conc{B}\sqsubseteq \conc{C}\}$ and the query
% $?X\sqsubseteq C$. Then the evaluation of this query over $\ont$ would
% include the concepts $\conc{A}$ and $\conc{B}$. However, they would
% include a multitude of possibly unwanted, redundant query results, e.g.,
% $\{\conc{A}\sqcap \conc{A}, \conc{A}\sqcap \conc{A}\sqcap \conc{A},
%  \ldots\}$. Restricting query evaluation to syntactic entailment addresses the
% issue of capturing results at the cost of
% ontology reasoning, yielding only $\{\conc{A}\}$ in the above example. 
% %This is clearly highly restrictive and likely not always desirable. 
% Finding suitable entailment relations between these two extremes would be highly beneficial to GBoxes and is as such an important direction for future work.

\paragraph{Entailment problems for ontologies with Gboxes} 
%In this paper we showed, among other things, how to compute an expansion given a GBox and an ontology. While doing so clearly allows us to perform standard reasoning tasks on the new ontology, a question of interest is whether and when reasoning on the input ontology and GBox directy, without computing an expansion, can provide gains in reasoning efficiency. 
%In particular, a GBox-specified ontology can have, if an expansion is computed, many very similar axioms. Reasoning on the generators that produced them directly ought to be more efficient in such cases.

%Furthermore, there are plenty of reasoning tasks that are about GBoxes themselves. A lot of these reduce to classic ontology reasoning tasks: For example, checking whether a single generator $g : T_B \rightarrow T_H$ always leads to inconsistency is equivalent to checking whether $T_B \cup T_H$ is inconsistent. 
%Likewise, checking if $g$ can ever fire, that is, generate axioms, is equivalent to checking whether $T_B$ is consistent. 
%This generalizes to similar questions over entire GBoxes: To check whether there exists an ontology $\ont$ such that every generator $g$ in a GBox $G$ fires, it suffices to check that the union of the generators' bodies is consistent. 

%On the other hand, we have global properties of GBoxes that do not reduce to questions about the individual templates. For example, do two GBoxes $G_1$ and $G_2$ specify equivalent ontologies? While \Cref{sec:containment} contains some results about such problems, we believe there is more to do here.

The expansion of a Gbox over an ontology is itself an ontology and can be used as such for standard reasoning tasks. A question of interest is whether/how reasoning on the input ontology and GBox directly, without computing an expansion, can improve reasoning efficiency.

Furthermore, there are plenty of reasoning tasks about GBoxes which naturally reduce to reasoning tasks over ontologies. For example, checking whether a single generator $g \colon T_B \rightarrow T_H$ always leads to inconsistency is equivalent to checking whether $T_B \cup T_H$ is inconsistent. This generalizes to similar questions over entire GBoxes: To check whether there exists an ontology $\ont$ such that every generator $g$ in a GBox $G$ fires, it suffices to check that the union of the generators' bodies is consistent.

However, there are also global properties of Gboxes that do not reduce to individual templates. For example, do two GBoxes $G_1$ and $G_2$ specify equivalent ontologies? While \Cref{sec:containment} contains some results about such problems, we believe there is more to do here.

\paragraph{Extensions to generators} Another area of future work is motivated by our preliminary analysis
of \emph{logical} ontology design patterns \cite{gangemi2009ontology}. We found that a number of rather
straightforward, seemingly useful such pattern require some form of
ellipses and/or maximality. Consider, for example, the role closure
pattern on the role $\conc{hasTopping}$: if $\ont$ entails that
$\conc{MyPizza}\sqsubseteq \exists\conc{hasTopping}.X_1\sqcap \ldots
\exists\conc{hasTopping}.X_n$ and $n$ is maximal for pairwise incomparable $X_i$, then
we would like to automatically add  $\conc{MyPizza}\sqsubseteq
\forall\conc{hasTopping}.(X_1\sqcup \ldots \sqcup X_n)$. Extending
generators to capture some form of ellipses or unknown number of
variables and maximality conditions on substitutions for variables
will be part of future work. 

For GBoxes to be indeed intention revealing, we will also support
named generators and named sets of axioms in the body or the head of
generators, as in OTTR \cite{OTTR-ISWC18}.

%%% Local Variables:
%%% mode: latex
%%% TeX-master: "main"
%%% End:

\bibliographystyle{plain}
\bibliography{references}
\end{document}